\documentclass[aps,nofootinbib,groupaddress,superscriptaddress,reprint,amsmath,amssymb,graphicx,onecolumn,prx]{revtex4-2}
\usepackage{physics}
\usepackage{bm}
\usepackage{bbold}
\usepackage{graphicx}
\usepackage[toc,page]{appendix}
\usepackage{float}
\usepackage{color}
\usepackage{algorithm}
\usepackage{algpseudocode}
\usepackage{amsthm}
\usepackage{csquotes}

\newtheorem{theorem}{Theorem}[section]
\newtheorem{lemma}[theorem]{Lemma}
\begin{document}
\preprint{APS/123-QED}

\title{Information-directed sampling for bandits: a primer}
%Heuristic strategies to balance exploration and exploitation in partially observable Markov decision processes

\author{Annika Hirling}
\affiliation{Institute of Physics, School of Basic Sciences, École Polytechnique Fédérale de Lausanne - EPFL, Lausanne, Switzerland}
\affiliation{Quantitative Life Sciences section, The Abdus Salam International Center for Theoretical Physics (ICTP), Trieste, Italy}
\author{Giorgio Nicoletti}
\affiliation{Quantitative Life Sciences section, The Abdus Salam International Center for Theoretical Physics (ICTP), Trieste, Italy}

\author{Antonio Celani}
\affiliation{Quantitative Life Sciences section, The Abdus Salam International Center for Theoretical Physics (ICTP), Trieste, Italy}
\affiliation{Department of Oncology, Universit\`a degli Studi di Torino, Italy}

\begin{abstract}
\noindent The Multi-Armed Bandit problem provides a fundamental framework for analyzing the tension between exploration and exploitation in sequential learning. This paper explores Information Directed Sampling (IDS) policies, a class of heuristics that balance immediate regret against information gain. We focus on the tractable environment of two-state Bernoulli bandits as a minimal model to rigorously compare heuristic strategies against the optimal policy. We extend the IDS framework to the discounted infinite-horizon setting by introducing a modified information measure and a tuning parameter to modulate the decision-making behavior. We examine two specific problem classes: symmetric bandits and the scenario involving one fair coin. In the symmetric case we show that IDS achieves bounded cumulative regret, whereas in the one-fair-coin scenario the IDS policy yields a regret that scales logarithmically with the horizon, in agreement with classical asymptotic lower bounds. This work serves as a pedagogical synthesis, aiming to bridge concepts from reinforcement learning and information theory for an audience of statistical physicists.
\end{abstract}

\maketitle

\section{Introduction}

% \begin{displayquote}
% It must, in all justice, be admitted that never again will scientific life be as satisfying and serene as in days when determinism reigned supreme. In partial recompense for the tears we must shed and the toil we must endure is the satisfaction of knowing that we are treating significant problems in a more realistic and productive fashion.
%   ―Richard Bellman, Adaptive Control Processes: A Guided Tour (page 129f, 1961)
%   \end{displayquote}
\begin{displayquote}
“Knowing ignorance is strength.  Ignoring knowledge is sickness.”
  ―Lao Tsu, Tao Te Ching
\end{displayquote}

\noindent Understanding how to make good decisions in uncertain and partially observable environments is the key to success for humans, animals and machine alike. How can we make a profitable use of the limited knowledge we have about the world? How can we increase our knowledge without indulging in idle curiosity? How can we maximize our gain without being greedy and close-minded? 

The Multi-Armed Bandit (MAB) problem -- which describes situations where a decision-maker must choose repeatedly among different options, each giving uncertain rewards -- serves as the canonical mathematical framework for studying this kind of dilemmas \cite{robbins1952aspects,Lattimore_Szepesvári_2020}. Although bandit problems are often characterized as a simplified model of the broader Reinforcement Learning (RL) landscape, they effectively encapsulate the central tension of sequential learning: the trade-off between exploration and exploitation \cite{Sutton_Barto_2018}. The agent must continuously decide whether to adhere to the option currently believed to be the best (exploitation) or to sacrifice immediate returns to gather data on less-tested options (exploration). While the problem formulation is simple, the optimal resolution of this tension is mathematically profound and often computationally intractable. 

Despite the intuitive appeal of the exploration-exploitation trade-off, the question of how to make this notion operational still stands. We know that balancing information acquisition and regret (the shortfall with respect to the best choice in hindsight) is key, but the precise, optimal weighting function for general settings is not known in general.

Information Directed Sampling (IDS) policies represent a class of heuristic algorithms that make this trade-off explicit \cite{russo2014learning,russo2018learning,min_information-theoretic_2023,kirschner2020information}. IDS strategies are defined by an optimization objective that balances immediate regret against information gain. At each step of the task, the policy evaluates the ``cost" of exploration in terms of forgone reward against the ``value" of the resulting reduction in uncertainty. 

 Comparing the performance of interpretable IDS strategies against a theoretically optimal policy can offer significant insight into the nature of efficient learning. Unfortunately, in general bandit problems -- despite their deceptive simplicity -- rarely allow for closed-form solutions. The calculation of Gittins indices \cite{gittins1979bandit} 
 %,gittins2011multi}\textcolor{red}{I am not sure about the second reference} 
 or the solution of the Bellman equation for continuous belief spaces \cite{krishnamurthy2016pomdp} is often feasible only through approximation. Consequently, we must look to the simplest non-trivial cases.

To this end, we treat the two-state Bernoulli armed bandit problem as the "hydrogen atom" of decision-making under uncertainty. Just as the hydrogen atom provides a tractable quantum mechanical system that guides our understanding of the physics of more complex elements, the two-state Bernoulli bandit is simple enough to allow for the computation of the optimal strategy, yet complex enough to exhibit the non-trivial dynamics of belief updates and information value.

The primary goal of this work is to study a specific sub-class of IDS strategies and compare them directly to the optimal policy within this controlled setting. By restricting our analysis to a problem simple enough to solve, we can rigorously quantify how closely the explicit information-regret ratios of IDS approximate the optimal decision path. This comparison allows us to probe the efficacy of Information Directed Sampling, validating its design principles while highlighting the gaps between heuristic efficiency and theoretical optimality.

This is a largely a pedagogical article.  It would be a daunting task to even provide a cursory overview of the work done in this field. The reader interested in delving deeper in the algorithmic approaches to bandit problems should definitely read the comprehensive and beautiful book by Lattimore and Szepesvari \cite{Lattimore_Szepesvári_2020}. Many of the theoretical concepts discussed herein have been introduced in prior literature. Our aim is to synthesize and scrutinize them in a language that can be accessible to statistical physicists who are interested in problems that involve decision-making, stochastic control and information theory.

% \begin{itemize}
%     \item Decision making is everywhere from computer science, economics, management science, to biology
%     \item Although bandit problems are a simplified model reinforcement learning problems, they still contain the central tension between exploration and exploitation. Information directed sampling policies are a way to make this tradeoff explicit. Indeed they are defined so as to find a tradeoff between immediate reward and information gain at each step of the task. The tradeoff is a key element in decision making, but we do not know how it is optimally handled. IDS strategies are easily interpretable, in the sense that we know how the decisions are made. Therefore, comparing their performance to the optimal one can offer some insight as to how to optimally handle the balance. 
    
%     \item This is a pedagogical article, many of the theoretical contents of this article have been introduced previously.
%     \item The two state Bernoulli armed bandit problem is something like the hydrogen atom
%     \item The goal of this work is to study a sub-class of IDS strategies and compare them directly to the optimal policy. For this purpose, we study a problem that is simple enough to compute the optimal strategy
%     \item In general even bandit problems don’t allow a closed form solution. 
% \end{itemize}

%The two-state Bernoulli bandit considered here is simple enough to allow accurate numerical computations of the optimal value function and policy, and, in the cases that we discuss below, even analytical expressions. 

\section{Partially Observable Markov Decision Processes with two states}
\noindent In this section, we introduce the general framework of a partially observable Markov decision process, where an agent takes actions in an unknown environment and receives partial observations in response. We focus on the simple scenario where the environment can be described by two discrete states, which has important applications in the class of problems known as ``bandit problems''.

\begin{figure}[t]
    \centering
    \includegraphics[width=1\linewidth]{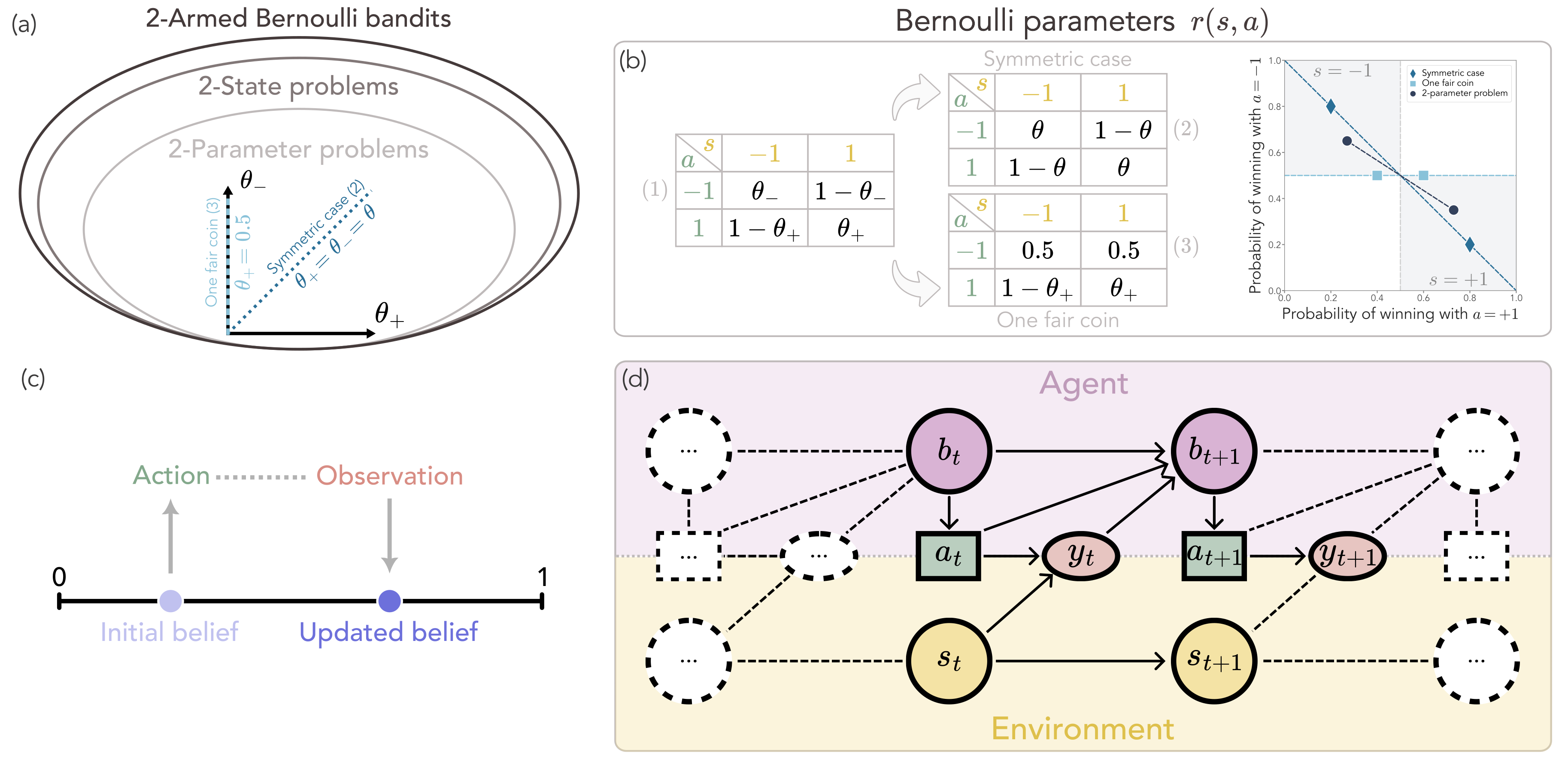}
    \caption{(a-b) Classes of two-armed Bernoulli bandit problems. The outer square contains all two-armed bandit problems where the reward distributions are Bernoulli. The outer circle contains all the two-state problems, which in general are described by four different parameters. The inner circle contains the problems parametrized by $\theta_+,\theta_-$ -- table 1 of panel (b). The special symmetric and one fair coin cases are shown in the $(\theta_+,\theta_-)-$space and in tables 2 and 3 of panel (b), respectively. (c) Schematic sketch of a belief update for our two-state problem, where the belief is defined on a one-dimensional space, and of the evolution of a POMDP in time, showing the belief $b$, state $s$, observations $y$, and actions $a$ for two steps of the process.}
    \label{fig:placeholderSchema2armedBB_POMDP}
\end{figure}

\subsection{Two-armed Bernoulli Bandits}
\noindent Consider the problem of a gambling agent who has to choose between two biased coins and wants to maximize their total gain. Each coin has a fixed but a priori unknown probability of landing on the ``winning'' face, which yields a reward. The gambler faces the following dilemma. On one hand, it should gather information about both coins so as to identify the best one with reasonable certainty and then be able to act optimally. On the other hand, they want to avoid accumulating too much loss during this exploration.

This decision process is part of a larger class of problems where an agent repeatedly chooses between two actions and collects rewards, which can be thought of as either ``win'' or ``lose'', and they are described by a random variable that follows a Bernoulli distribution for each action. The rewards effectively represent the observations that the agent receives from an unknown environment -- i.e., the agent does not have access to the pair of Bernoulli parameters of the coins. In other words, the agent cannot directly observe the state variable that specifies the environment, and this partial observability makes the problem non-trivial. 

\subsubsection{Two-state problems}
\noindent To simplify the problem, we now assume that the environment can only be in one of two possible -- but unobservable -- states. The problem is fully described by the following elements:
\begin{itemize}
	\item a state of the environment $s\in\{-1,1\} = \{\text{“Coin 1 is better”},\ \text{“Coin 2 is better”} \}$;
	\item an action $a\in\{-1,1\} = \{\text{“Choose Coin 1”},\ \text{“Choose Coin 2”} \}$;
	\item an observation $y\in\{0,1\}= \{\text{“Loss”},\ \text{“Win”} \}$;
	\item a reward $r(y)\equiv y$.
    %\item \textcolor{magenta}{say again that the system is stationary?}
\end{itemize}
Thus, the probability of receiving an observation is described by the Bernoulli distribution
\begin{equation}
    p(y | s, a) = [\theta(s, a)]^y [1-\theta(s,a)]^{1-y}
\end{equation}
where $\theta(s,a)$ represents the probability of winning after taking the action $a$ while the environment is in state $s$. We stress here that we exclusively focus on problems where the state does not change -- i.e., the better coin is always the same. Thus, an instance of this two-state problem is generally specified by the four Bernoulli parameters $\theta(s,a)$, one per state-action pair, which also correspond to the expected rewards $r(s, a) = E[R|s, a]$.

This parameter space is still relatively large and not straightforward to interpret. Since the purpose of this work is to study the trade-off between seeking an immediate reward and gathering information, we aim to further reduce the number of parameters while still making it possible for the actions to have different expected gains of information. Intuitively, an action bears a large information content if its reward distribution significantly differs from one state to the other. In our problem, this can be quantified by the difference in the Bernoulli parameters for the two states which we will call the bias $\delta_a=r(s=-1,a)-r(s=1,a)$. By focusing on a subclass of problems in which the Bernoulli parameters take the values shown in table 1 of Figure \ref{fig:placeholderSchema2armedBB_POMDP}b, where $\theta_-,\theta_+>0.5$, we allow the actions to have different biases namely, $\delta_-=2\theta_- -1$ for action $-1$ and $\delta_+=2\theta_+ -1$ for action $1$. 
\subsubsection{Special cases of two-state problems}
\noindent Besides the general case, we consider two particular choices of parameters that are particularly relevant. If $\theta_+ = \theta_- = \theta$, we have a \textit{symmetric} bandit problem, where the knowledge of one coin automatically translates into knowledge of the other. Indeed, the bias is the same for both actions, $\delta_-=\delta_+$. As we will see, this can be quantified exactly by the fact that the expected information gain is the same for both actions.

Another relevant example corresponds to the case where one of the two coins has zero bias, i.e., of \textit{one fair coin}. In particular, we take $\theta_- = 0.5$, so that $\delta_- = 0$, which implies that no information can be gained when choosing the $a = -1$ action. Thus, the agent now has to choose whether to play conservatively by playing the fair coin or attempt to gather information about the other coin and risk losing rewards. In this sense, in terms of information gain, these two subclasses constitute two extreme cases of the general problem where $\theta_+$ and $\theta_-$ are different (see Figure \ref{fig:placeholderSchema2armedBB_POMDP}, tables 2 \& 3).

\subsection{Bayesian POMDPs}
\noindent The theoretical framework that describes these problems is that of partially observable Markov decision processes (POMDP). A sketch of a POMDP is shown in Figure \ref{fig:placeholderSchema2armedBB_POMDP}c. The process is composed of the environment, the agent, and the interface between them. At each step, the environment is in a latent state that is not directly observable by the agent. Therefore, to make informed decisions, the agent must keep an (often compressed) memory of its past interactions with the environment, which serves as an indication of its current state. In the Bayesian setting, this memory is called the \textit{belief} $b(s)$, which is a probability distribution over the states corresponding to the posterior distribution given a prior and the sequence of interactions between the agent and the environment. %, $b(s)=\mathbb{P}(S=s|\text{"history"})$.
In the two-state problem that is studied here, the belief can be parametrized as 
\begin{equation}
    b(s)=\frac{1+s\beta }{2},
    \label{eq:belief_parametrization}
\end{equation}
which means in particular that $b(1)=1 \iff \beta=1$ and $b(-1)=b(1)=1/2\iff \beta=0$.

The process then goes as follows: the agent selects an action $a$ and consequently makes an observation $y$, which is assumed to be drawn from a probability distribution that depends on the state of the environment and the selected action. In our case, this is a Bernoulli distribution $y\sim p(y|sa)=\text{Bernoulli}(\theta(s,a)))$. In general, the environment may then transition to a new state based on a stochastic process -- e.g., due to the fact that the agent's action has affected or changed parts of it. In this work, however, the environment is assumed to be stationary. The agent has a model of the environment, meaning that it knows the distribution of the observation for each state-action pair. Therefore, given a prior belief distribution $b(s)$, an action $a$, and the consequent observation $y$ the agent can update its belief according to Bayes' theorem
\begin{align}
    b_{ay}'(s)=\frac{b(s)p(y|s,a)}{\sum_{s'}b(s)p(y|s',a)}
    \label{eq:belief_update}
\end{align}
so that the new belief $b_{ay}'(s)$ encodes the new observation made by the agent. Finally, the agent has to select the action at each step. This is done through a belief-dependent probability distribution over the actions called the \textit{policy}, 
\begin{equation}
    \pi(a|b)=p(A=a|b)
\end{equation}
which constitutes the control of the agent. In general, the optimal policy is deterministic, mapping each belief to a single action. However, heuristic policies can be stochastic, so that, at least in some regions of the belief space, the policy can differ from delta distribution concentrated on a single action.
\subsection{Bellman Optimality equation}
\noindent The agent's objective is to select a policy that maximizes its total gain, either over a fixed time horizon $T$ or over an infinite time horizon. In this work, we focus on the latter case, and the cumulative reward is a random variable that can be expressed as:
\begin{equation}
    G = \sum_{t=0}^\infty \gamma^t R_t
\end{equation}
where $R_t$ is the reward at time $t$ and $\gamma\in\left[0,1\right]$, called the \textit{discount factor}, ensures that the sum remains finite. A consequence of discounting is that the weight attributed to the reward decreases geometrically. Although there is no fixed time horizon, this results in an effective one, $T_\mathrm{eff}=1/(1-\gamma)$, which can be seen as a measure of the characteristic duration of the problem.

Assuming the agent follows a given policy $\pi$ and there is a prior distribution $b$ over the initial state, i.e., an initial belief of the agent, one can write the expectation of the discounted sum with respect to this policy as
\begin{equation}
    v_\pi(b)=\mathbb{E}_\pi\left[\sum_{t=0}^\infty \gamma^tR_t\Big\vert b_0=b\right]
    \label{eq_def:value_function}
\end{equation}
where $v_\pi(b)$ is a function of the initial belief distribution called the \textit{value} of the policy $\pi$.
%However, this expectation cannot be computed in practice as it is a sum over infinitely many events in the future.
By splitting the sum into the immediate reward and all the subsequent ones, one can obtain a recursive equation known as the \textit{Bellman's recursion equation} that implicitly defines the value function:
\begin{align}
    v_\pi(b)=\sum_a \pi(a|b)\sum_{y}p(y\vert s,a)\left[r(y)+\gamma v_\pi(b'_{ay})\right]
    \label{BellmanEqPOMDP}
\end{align}
where $r(y)$ is the function mapping the observation to the reward and $b'_{ay}$ is the updated belief (Eq.~\eqref{eq:belief_update}). In our case, we simply have $r(y)\equiv y$. We stress that this form of the Bellman's equation holds when the state of the environment does not change, which is the only case we consider here.

By definition, the policy $\pi^*$ that maximizes the value function is the optimal one. In Eq.~\eqref{eq_def:value_function}, if we denote with $v^*(b)$ the associated value function, after separating the right-hand side, one obtains the \textit{Bellman's optimality equation}:
\begin{align}
    v^*(b)=\underset{a}{\text{max}}\left[\sum_{y}p(y\vert s,a)\left[r(y)+\gamma v^*(b'_{ay})\right]\right] \; .
    \label{BellmanOptimalityEqPOMDP}
\end{align}
For relatively simple problems, like the two-state, two-armed bandits studied here, Equations \ref{BellmanEqPOMDP} and \ref{BellmanOptimalityEqPOMDP} can be used to compute the value function of a policy numerically \cite{smallwood1973optimal} (see Figure~\ref{fig:Illustration_Value_Regret}. The standard algorithm to solve Bellman's optimality equation is \textit{value iteration}, which is guaranteed to converge to the solution of the equation. The pseudocode and proof of convergence for this method are provided in the Appendix \ref{app:value_iteration}.

\subsection{Regret of a policy}
\noindent The Bellman's equation tells us that, to each state $s$, there is an associated optimal action $a^*(s)=\text{argmax}_a\;r(s, a)$. Since, in our case, the state of the system is stationary, if the agent knew the state, it could achieve a maximum expected discounted reward of $V_\mathrm{MDP}^*(s)=r(s,a^*(s))/(1-\gamma)$ by always selecting the optimal action. $V_{MDP}^*(s)$ represents the optimal value at state $s$ of the underlying \textit{fully observable} Markov decision process (MDP). The underlying MPD is nothing but the equivalent problem with the agent's uncertainty about the state removed. For an arbitrary belief $b(s)$, the optimal MDP value is given by the weighted sum of the optimal MDP values:
\begin{equation}
    \overline{v}(b)=\sum_s b(s)V_\mathrm{MDP}^*(s) \; .
    \label{eq:opt_MDP_value}
\end{equation}
Given that in a POMDP, the agent never observes the actual state of the system, it cannot hope to perform better than another agent who does. A relevant way to measure the performance of a POMDP policy is therefore to subtract its value function from this optimal MDP-value. This quantity is called the \textit{regret} of a policy:
\begin{align}
    \mathcal{R}(b,\pi)=\overline{v}(b)-v_\pi(b) \; .
\end{align}
An example of value and regret in a generic two-state problem is shown in Figure \ref{fig:Illustration_Value_Regret} for two different discount factors $\gamma$. At the edges of the belief space, the agent knows the environment's state with certainty, which, for a stationary problem, means that it can select the optimal action at every future step. Therefore, the optimal POMDP and MDP values coincide at these points, and the regret vanishes. The value and regret both increase with $\gamma$, since the effective time horizon of the problem scales like $1/(1-\gamma)$. Thus, for a higher discount factor, the agent can collect more reward but also accumulate more regret. 

\begin{figure}
    \centering
    \includegraphics[width=1.\linewidth]{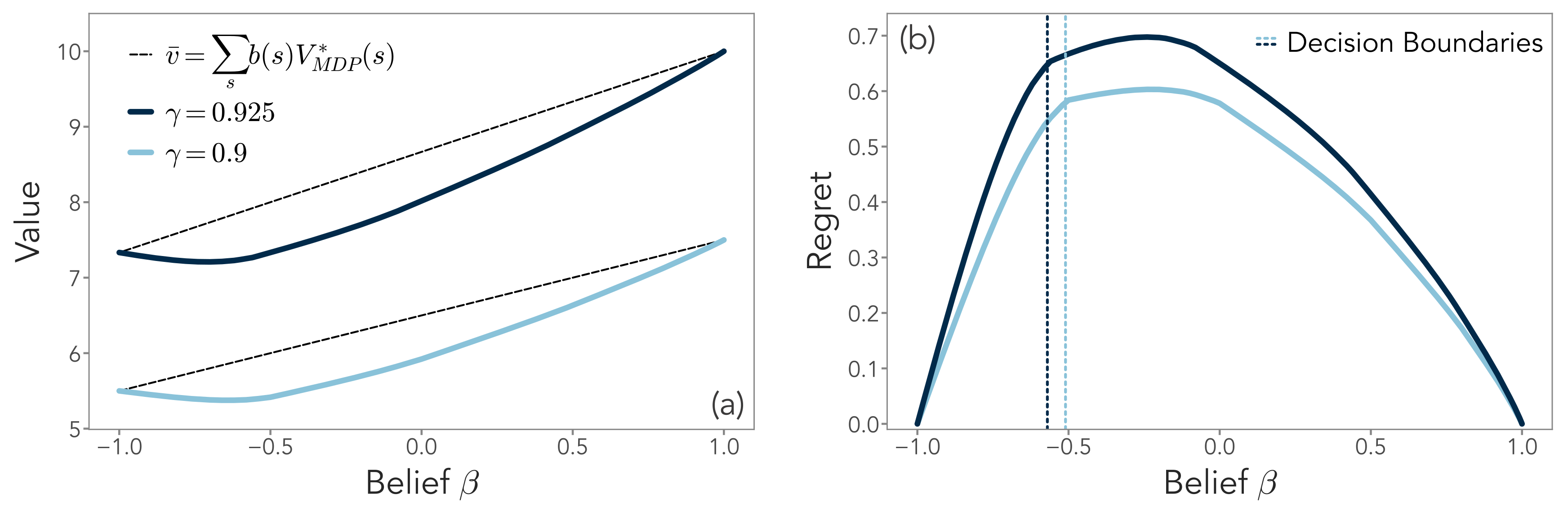}
    \caption{Examples of value functions and corresponding regrets. (a) Optimal value functions at different $\gamma$ for $\theta_+ = 0.7$ and $\theta_- = 0.55$. The dotted line corresponds to the value of the underlying Markov decision process (MDP). (b) Regrets for the values in panel (a). The dotted line corresponds to the decision boundary, where the optimal policy switches from one action to the other.}
    \label{fig:Illustration_Value_Regret}
\end{figure}

The decision boundary $\beta_c$ separates the belief space into a left and a right region where the optimal policy respectively selects action $-1$ and $1$. In the present example, the boundaries are shifted. This means that the optimal policies select action $1$ even for beliefs $\beta_c<\beta<0$ where action $-1$ is believed to be more rewarding. As we will see, this is an example of favoring exploration over exploitation, which is why the boundary is even more shifted for larger $\gamma$, where the effective horizon is longer and information about the state is more valuable.

\section{Information directed sampling}
\noindent Decision-making problems often require a policy that balances exploitation -- playing the action that is currently believed to be optimal to maximize the rewards -- and exploration -- to attempt to find better actions. However, in general, tackling the issue of finding such optimal policies directly from Bellman's optimality equation (Eq.~\eqref{BellmanOptimalityEqPOMDP}) is computationally unfeasible, as the state and action spaces may become extremely large. In this scenario, one can either resort to approximate methods or design a heuristic policy and test its performance a posteriori. Here, we focus on a class of heuristic policies known as \textit{Information Directed Sampling} (IDS), introduced by Russo and Van Roy \cite{russo2014learning}. These policies aim to balance exploration and exploitation by minimizing at each step the ratio between the expected one-step regret and a measure of the expected information gain. 

Beyond its intuitive appeal, one can prove that IDS policies tighten some bounds on the total regret. These bounds have previously been derived for the case of a finite time horizon and no discount factor \cite{russo2014learning,min_information-theoretic_2023}. Here,  we first generalize these bounds for the discounted, infinite-horizon setting, and then show how it is possible to introduce an explicit tuning parameter $\alpha$ that allows the resulting IDS policy to move between a greedy, exploitative behavior to a more exploration-oriented one. We name this policy as the IDS($\alpha$) one.

Let $\Delta_\pi$ be the \textit{one step regret},
\begin{equation}
    \Delta_\pi(b) =\sum_{a,s}\left[\text{max}_{a'} r(s,a')- r(s,a)\right]\pi(a|b) b(s)
\end{equation} 
which measures the difference between the expected reward that could be obtained by knowing the state $s$ and playing the optimal action, and the expected reward following policy $\pi$. We introduce the one-step information measure as
\begin{align}
    I_\pi(b)=(1-\gamma)H(b)+\gamma\left(H(b)-\sum_{ay}\pi(a|b)p_b(y|a) H(b'_{ay})\right)
    \label{eq:infofunction}
\end{align}
where $H(b)$ is the Shannon entropy of the belief distribution, $H(b)=-\sum_sb(s)\log b(s)$. Crucially, the second term in Eq.~\eqref{eq:infofunction} is the expected reduction of entropy following policy $\pi,$ which corresponds to the mutual information between the selected action and the state, multiplied by the discount factor $\gamma$. Then, by building a bound on the cumulative discounted cost incurred by a policy $\pi$, we define its \textit{information ratio} as
\begin{align}
    \Psi_\alpha(\pi)=\underset{b}{\text{sup}}\Psi_\alpha(\pi, b), \quad  \Psi_\alpha(\pi, b)=\frac{\Delta(b)^{1/\alpha}}{I_\pi(b) ^{1/\alpha-1}}
    \label{eq:inforatioalpha}
\end{align}
where the supremum is taken over all beliefs and $\alpha\in[0,1]$ is a tuning parameter stemming from Young's inequality. It can be shown that for any policy $\pi$ (Appendix \ref{app:ids_bounds}), the total regret satisfies the bound
\begin{equation}
    \text{Regret}(\pi,b)\leq\left(\frac{\Psi_\alpha}{1-\gamma} \right)^\alpha H(b)^{1-\alpha}
\end{equation}
where the only dependence on the policy is in the ratio  $\Psi_\alpha(\pi)$. At each time, the IDS($\alpha$) policy is defined as 
\begin{equation}
    \pi^t_{\mathrm{IDS}(\alpha)}=\underset{\pi}{\text{argmin}}\;\Psi_\alpha^t(\pi)
\end{equation}
The core idea of IDS is to minimize the information ratio over the policy at each step in order to keep it as low as possible, and ultimately get a tighter bound on the regret. The policies can then be tested with different values of $\alpha$, and their performances compared to one another as well as to the optimal policy. Importantly, in some cases, we are able to make such a comparison directly with the exact solution of our bandit problem, allowing for a comprehensive characterization of their performance. Furthermore, the parameter $\alpha$ can be qualitatively interpreted as a balance between exploration and exploitation. Indeed, if $\alpha = 1$, IDS($1$) amounts to minimizing the one-step regret, i.e., to act greedily and exploit the immediate reward. On the other hand, when $\alpha \to 0$, the policy weighs in information, and thus enables the agent to explore actions that it currently believes are sub-optimal to learn more about their reward outcomes.

It should be noted that the information function (Eq.~\eqref{eq:infofunction}) is different from the one that was introduced previously \cite{russo2014learning} for the undiscounted setting, which was simply the mutual information. In the limiting case $\gamma\to1$, however, the information function reduces to the mutual information and the functions coincide. Another argument in favor of the new definition in the discounted setting is that, in the other limit $\gamma\to0$, we have $I_\pi(b)\to \mathrm{const}$. This implies that minimizing the information ratio reduces to minimizing the regret, so that the IDS-policy becomes the greedy policy. This is expected in this case, where only the first reward matters and the optimal strategy aims to maximize the expected reward in this single step.

\section{Symmetric Bernoulli Bandits}
\noindent We now exploit the simplicity of the two-state Bernoulli bandit problems considered here to gauge IDS performance with respect to the optimal policy, which can be accurately computed numerically or even analytically. In this section, we start by presenting the results for the special case of symmetric two-state two-armed bandits, in which the Bernoulli parameters are $\theta_\pm = \theta$ (see Figure \ref{fig:placeholderSchema2armedBB_POMDP}b, table 2). We derive an analytical solution for the optimal value, the $\mathrm{IDS}(\alpha)$ policies, and study the properties of the regret as a function of both $\theta$ and the discount parameter $\gamma$.

%The first subsection focuses on the exact analytical solution that was derived for this particular case. The second section presents different aspects of the problem through numerical results and comparison with the analytical solution (Figure \ref{fig1:placeholder}). First, we show examples of the optimal and IDS($\alpha=0$) regret in two specific problem instances (a), illustrating that IDS performs optimally in this case. Then we show an overview of the relative difficulty of the possible instances of the symmetric problem (b), which is quantified by the maximum of the optimal regret in each instance. Finally, we show that the maximal optimal regret converges to a constant as $\gamma\to1$, through numerical results and a low-order expansion of the analytical solution around $\gamma=1$ (c).
 
\subsection{Analytical solution for the optimal value}
\noindent In the symmetric case, the optimal value function can be found analytically. We start from the Bellman optimality equation (Eq.~\eqref{BellmanOptimalityEqPOMDP}) and solve it separately for the belief domains where $a = \pm 1$ are the optimal actions, and then match the two solutions at the decision boundary. Importantly, due to the symmetry, the decision boundary is known a priori and it is given by $\beta_c=0$. This also means that the optimal policy is the greedy policy. For both domains, we make the following ansatz:
\begin{equation}
    v^*(b)=\frac{r_{a}^Tb}{1-\gamma}+C\prod_s b(s)^{\zeta(s)} \;\;\;\;\;\;\text{with}\;\;\;\;\;\; \sum_s\zeta(s)=1
    \label{eq:ansatz}
\end{equation}
where $r_{\pm1}^Tb=\sum_sb(s)r(s,a)=[1 + a \beta (2 \theta -1)] / 2$ corresponds to the expected immediate reward when selecting action $a=\pm1$ at belief $b$. Given the constraint on the exponent, we can write $\zeta(-1)\equiv\zeta$ and $\zeta(1)\equiv1-\zeta$.
The solutions for the optimal value are then\footnote{Throughout the paper we will slightly abuse the notation, meaning $v^*(\beta)\equiv v^*(b(\beta))$.}
\begin{equation}
    v^*(\beta)=
    \begin{cases}
        v^*_+(\beta), & \text{if } \beta \geq 0 \\
        v^*_-(\beta), & \text{if } \beta < 0
    \end{cases}, \; \; \; \; \; \; \; \;
    v^*_\pm(\beta)=\frac{r_{\pm1}^Tb}{1-\gamma}+\underbrace{\frac{\gamma (1 - 2\theta)^2}{(1 - \gamma) \sqrt{1 - 4\gamma^2\theta(1 - \theta)}}}_{C} \frac{(1+\beta)^{\zeta_\pm} (1-\beta)^{1 - \zeta_\pm}}{4}
\end{equation}
where the exponents are given by %\textcolor{red}{(CHECK THE $\pm$)}
\begin{equation}
    \zeta_{\pm}=\frac{\ln \left( \frac{1 \pm \sqrt{1 - 4\gamma^2\theta + 4\gamma^2\theta^2}}{2\gamma\theta} \right)}
    {\ln(1 - \theta) - \ln\theta} \; .
    \label{eq:exponent}
\end{equation}
After subtracting from the optimal value the one of the underlying MDP, the asymptotic behavior of the regret at $\beta = 0$ is captured by a first-order expansion around $1 - \gamma$ that gives
\begin{equation}
    \mathcal{R^*}(\beta=0)=\frac{1}{2\delta}-c(\theta)(1-\gamma)+o(1-\gamma)^2
    \label{eq: asymptotic_symmetric_regret}
\end{equation}
where $\delta=2\theta-1$ is the bias and the expression for the prefactor $c(\theta)$ is given in Appendix \ref{app:analytics}.

Remark that in this case, when the time horizon of the process goes to infinity, i.e. $\gamma \to 1$, the regret stays finite. This is a known consequence of the fact that in this case the ``confusion set'' of parameters that could be mistakenly taken as optimal is empty
\cite{agrawal1989asymptotically, panaganti2020bounded}. Another way of understanding this result is to note that, in this case, sampling an arm gives information also about the other one -- this is a structured bandit problem. As a result, a finite amount of information needs to be collected in order to distinguish between the bad and the good choice. 
As we will see below, this is not true in general, and in particular when one of the coins is fair, in which case the regret grows logarithmically with the time horizon.

\begin{figure}[t]
    \centering
    \includegraphics[width=\textwidth]{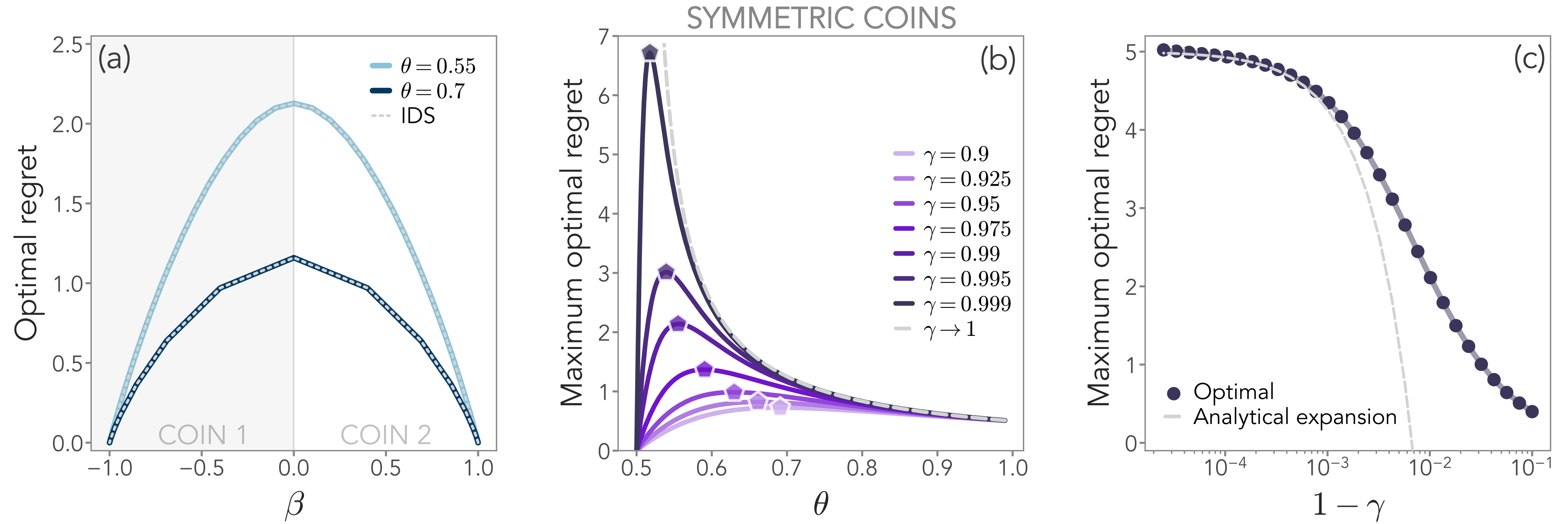}
    \caption{Results for the symmetric bandits case. (a) Optimal regret for two probabilities of winning $\theta=0.55$ and $\theta=0.7$, and the corresponding IDS regrets. IDS achieves optimal performance in this problem. Here, $\gamma=0.99$. (b) Maximum over $\beta$ of the optimal regret as a function of $\theta$ for different values of $\gamma$. The dotted curve corresponds to the asymptotic expansion of the regret (Eq.~\eqref{eq: asymptotic_symmetric_regret} for $\gamma\to1$. (c) Maximum of the optimal regret as $\gamma\to1$, which approaches a constant value. Here, $\theta=0.55$. The dotted line is the corresponding asymptotic expansion. Here, $\theta=0.55$.}
    \label{fig1:placeholder}
\end{figure}

\subsection{Optimal and IDS policies}
\noindent In Figure \ref{fig1:placeholder}a we show the regret for both the optimal policy and IDS. We find that, in these problems, the IDS policy performs optimally. Since, in this symmetric case, the bias $\delta$ is the same for both actions, it can be shown that the expected mutual information between the state $s$ and the action-observation pair is the same for both actions. As a consequence, the denominator of the information ratio (Eq.~\eqref{eq:inforatioalpha}) is also identical. Crucially, this means that IDS only minimizes the one-step regret and thus corresponds to the greedy policy, which is also the optimal one due to the symmetry of the rewards. Therefore, regardless of $\alpha$, IDS always coincides with the optimal policy.

\subsection{Properties of the regret} 
\noindent In the two displayed examples (Figure \ref{fig1:placeholder}a) the regret is higher for $\theta=0.55$. To get a complete overview of the behavior of the regret as a function of $\theta$, we compute the maximum of the regret as a function of $\theta$ and for different values of $\gamma$ (Figure \ref{fig1:placeholder}b). At $\theta = 0.5$, the two actions are identical, and thus the expected regret is always zero. On the other hand, when $\theta=1$, the regret goes to $0.5$ for all $\gamma$, since the belief update (Eq.~\eqref{eq:belief_update}) shows that the belief always collapses to $\pm1$. This means that the agent can deduce the state with certainty after playing only once. Once the belief has collapsed, the agent knows the optimal action and stops accumulating regret, regardless of $\gamma$. This remains approximately true for large $\theta$ as well, where the agent can quickly learn the state and essentially stops gathering regret after just a few steps.

For small $\theta$, instead, gathering information is difficult. On the flip side, there is not much loss in playing sub-optimal actions. As a function of $\theta$, the maximum of the regret peaks at intermediate values, where making decisions is most difficult. As $\gamma$ increases, this peak shifts to the left. Indeed, the weight associated with the discounted rewards decreases geometrically, so that future rewards essentially become negligible after an effective time-horizon $T_{\mathrm{eff}}\sim1/(1-\gamma)$. The difficulty of the problem then depends on whether the agent can distinguish the states on a timescale comparable to the effective horizon.
%In this case, by the time the agent learned the state and can act optimally, the rewards have practically no weight.
The time it takes to distinguish the states increases as $\theta$ decreases, which explains why the maxima shift to the left with increasing $\gamma$. In Figure \ref{fig1:placeholder}c, we show the scaling of the maximum regret for the optimal policy as a function of $1 - \gamma$. As the effective time horizon grows, the maximum regret approaches the constant found by expanding the analytical solution around $\gamma=1$ (Eq.~\eqref{eq: asymptotic_symmetric_regret}).
%The full first order development of the regret at $\beta=0$ is plotted and compared to the numerical values in the right panel (Figure \ref{fig1:placeholder}c), showing that it is relatively accurate for $\gamma\lessapprox10^{-3}$.

\section{One fair coin}
\noindent We now focus on another particular type of bandit problem, where one coin is fair ($\theta_- = 0.5$) and the other one is not ($\theta_+ \ge 0.5$, see Figure \ref{fig:placeholderSchema2armedBB_POMDP}b, table 3). As in the previous section, we first study Bellman's optimality equation and derive an approximate analytical solution. We compare the corresponding policy with those derived from IDS with $\alpha = 0$, which is now suboptimal but with a regret of the same order as the optimal one.

%In this section we present results analogous to the ones discussed in the previous chapter, this time for the second special case where one of the coins is fair, and the Bernoulli parameters take values parametrized by $\theta_+$ as shown in table 3 of Figure \ref{fig:placeholderSchema2armedBB_POMDP}c. In this case, the analytical solution that is presented is only approximately correct with increasing precision as $\theta_+\to0.5$. The next section shows again two examples of optimal and IDS(0) regret curves (Figure \ref{fig2:placeholder}a), showing that IDS is now sub-optimal but has a regret of the same order as the optimal one. Next we show how the regret varies over the possible problem instances (b) and find similar behavior as in the symmetric case. Finally, the asymptotic scaling of the optimal regret with $\gamma$ is again studied via development of the analytical solution and validated numerically (c), showing that the regret does not converge to a constant as it did in the symmetric case, but diverges logarithmically as $\gamma\to1$. \textcolor{magenta}{A numerical analysis of the asymptotic IDS regret showed that it appears to diverge logarithmically as well, with a slope higher than the optimal one.}

\subsection{Analytical solution}
\noindent In the one fair coin case, the decision boundary is not known a priori. However, since playing the fair coin does not provide information about the system, we must have that $\beta_c \le 0$. Therefore, an optimal agent will not always play greedily -- which corresponds to $\beta_c = 0$ as in the previous case -- given that 
%In the case where one of the coins is fair (see Figure \ref{fig:placeholderSchema2armedBB_POMDP}b, table 3), meaning that $\theta_1=0.5$, one can find an approximate analytical solution to the Bellman optimality equation which becomes increasingly accurate as $\theta_2$ approaches 0.5 as well. In the symmetric case, the optimal decision boundary was always located at $\beta_c=0$ which corresponds to playing greedily. In the case with the fair coin however, it is shifted as playing the fair coin does not provide information about the system and therefore there
it can be advantageous to choose the other coin even when the fair coin is believed to be more rewarding.

On the left side of the decision boundary, where the argmax of the Bellman equation is given by the fair coin ($a=-1$), the value function is constant. Since the expected reward is always $1/2$, we must have $v_-^*(\beta)=\frac{1}{2}(1-\gamma)$. On the right side, we make the same ansatz as before (Eq.~\eqref{eq:ansatz}) and obtain the piecewise approximate solution
\begin{equation}v^*(\beta)=
\begin{cases}
    \frac{1+\beta(2\theta-1)}{2(1-\gamma)}-\frac{\beta_c(2\theta-1)}{2(1-\gamma)}\frac{(1-\beta)^{\zeta_-}(1+\beta)^{1-\zeta_-}}{(1-\beta_c)^{\zeta_-}(1+\beta_c)^{1-\zeta_-}} , \; &\beta\geq \beta_c\\
    \frac{1}{2}(1-\gamma), \; &\beta<\beta_c
\end{cases}  
\end{equation}
where the $\zeta_-$ exponent is the same as before (Eq.~\eqref{eq:exponent}) and the decision boundary is given by
\begin{equation}
    \beta_c=-\frac{1}{2\zeta_--1} \; .
    \label{eq:db_fair_coin}
\end{equation}
The value of the optimal MDP is given by
\begin{equation}
    \bar{v}(\beta)=\frac{1+\beta}{2}\frac{\theta_+}{1-\gamma}+\frac{1-\beta}{2}\frac{1}{2}(1-\gamma)
\end{equation}
so that we have an explicit approximate expression for the regret as well. Since $\beta_c<0$, we evaluate the right branch of the solution and expand around $\gamma=1$ to obtain
\begin{equation}
    \mathcal{R}(\beta=0)=c_1+c_2\log(1-\gamma)+c_3(1-\gamma)\log^2(1-\gamma)+c_4(1-\gamma)\log(1-\gamma)+c_5(1-\gamma)
    \label{eq:analytical_development_faircoin}
\end{equation}
which is the asymptotic scaling of the regret at $\beta=0$. Contrary to the previous case of symmetric coins, the regret now scales logarithmically with $1-\gamma$ rather than going to a constant.

This dependence on the regret matches the asymptotic lower bound by Lai and Robbins \cite{lai1985asymptotically}.
\begin{figure}[t]
    \centering
    \includegraphics[width=\textwidth]{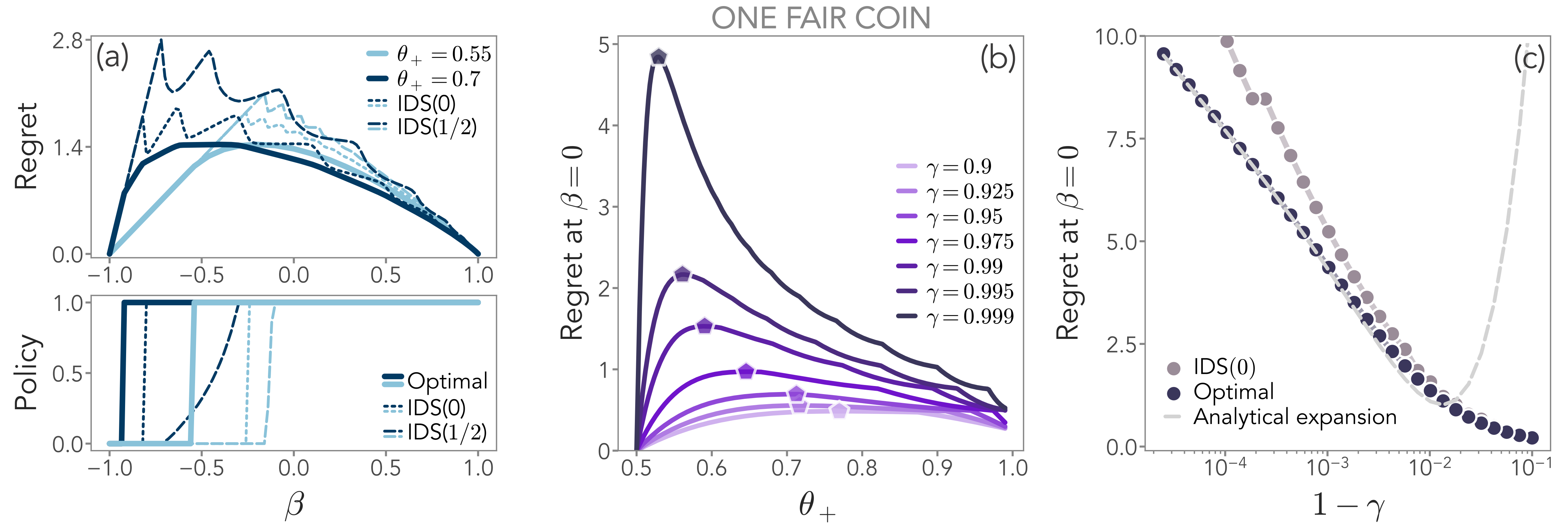}
    \caption{Results for the one fair coin case. (a) Regret and policy for the optimal agent and IDS for two probabilities of winning of the unfair coin, $\theta_+=0.55$ and $\theta_+=0.7$. IDS($0$) performs better than IDS$(1/2$), since its decision boundary is closer to the optimal one. Here, $\gamma=0.99$.(b) Optimal regret at $\beta = 0$ as a function of $\theta_+$ for different values of $\gamma$. (c) Regret at $\beta = 0$ as $\gamma \to 1$, which diverges logarithmically. IDS($0$) behaves similarly to the optimal policy. The dotted line is the corresponding asymptotic expansion. Here, $\theta_+=0.55$. }
    \label{fig2:placeholder}
\end{figure}

\subsection{Optimal and IDS($0$) policies}
\noindent In Figure \ref{fig2:placeholder}a we show both the regret and policy of the optimal agent and for IDS. In general, we find that the IDS policy is too greedy, while the optimal one favors exploration more. However, IDS with $\alpha = 0$ outperforms the one with $\alpha = 1/2$, placing the decision boundary closer to the optimal one. Therefore, for the rest of this section, we will focus exclusively on $\mathrm{IDS}(0)$. 

In the region where the policies select the fair coin, the value is constant, and therefore the regret increases linearly in both cases, but the IDS regret continues to do so between the decision boundaries, making it suboptimal. At the decision boundary of IDS, there is a sharp decrease in the IDS regret. This gap between the decision boundaries leads to a sawtooth behavior for the IDS regret. The next saw tooth can be understood as the set of beliefs which can be mapped to the region between the decision boundaries through the belief update equation (Eq.~\eqref{eq:belief_update}). This can also be seen from the Bellman equation (Eq.~\eqref{BellmanEqPOMDP}) as the value at each belief is equal to a sum containing the value at a smaller belief. Therefore, these wiggles can be interpreted as a propagation of the sub-optimality of IDS between the decision boundaries, and follow from the fact that the value of each belief depends on the value of neighboring ones.

\subsection{Properties of the regret} 
\noindent In Figure \ref{fig2:placeholder}b, we study the regret at $\beta=0$ as a function of $\theta_+$ and for different values of $\gamma$. Overall, the regret behaves similarly to the symmetric case (Figure \ref{fig1:placeholder}b), going to zero as $\theta_+\to0.5$. The location of the maxima tends to shift to the left as $\gamma$ increases for the same reasons as before. 
%The visible wiggles the curves display, mostly for high $\theta$ and $\gamma$, can be explained by the way the shape of the regret curves changes with $\theta$,  see Fig.~\ref{fig:AnalyticalSolution}. 

Since the scaling with $1-\gamma$ is now logarithmic, the regret diverges asymptotically (Figure \ref{fig2:placeholder}c). The expansion in Eq.~\eqref{eq:analytical_development_faircoin} proves to be a good approximation even at relatively large values of $1-\gamma$. Importantly, the regret obtained with the $\mathrm{IDS}(0)$ policy is of the same order as the optimal one, even if the policy is sub-optimal, showing that IDS is able to capture this fundamental feature of the problem. %(\textcolor{magenta}{talk about the scaling of the bounds here?})

\section{Asymmetric problem and performance of IDS($\alpha$)}
\noindent We now consider the case where the Bernoulli parameters $\theta_\pm$ are different (table 3 of Figure \ref{fig:placeholderSchema2armedBB_POMDP}b) and evaluate how  IDS performance varies with $\alpha$, which is, in general, a free parameter of our heuristic policy. To quantify it, we compute the maximum relative distance to the optimal regret, $\mathcal{R}^*(\beta)$:
 \begin{equation}
     \Delta \mathcal{R}(\alpha) = \max_\beta\left[\frac{\mathcal{R}_{\mathrm{IDS}(\alpha)}(\beta)-\mathcal{R}^*(\beta)}{\mathcal{R}^*(\beta)}\right]
 \end{equation}
which, in general, will depend on the specific problem parameters $(\theta_-, \theta_+)$, as we have already seen for the symmetric and the one fair coin case.
 
 \begin{figure}[t]
    \centering
    \includegraphics[width=\linewidth]{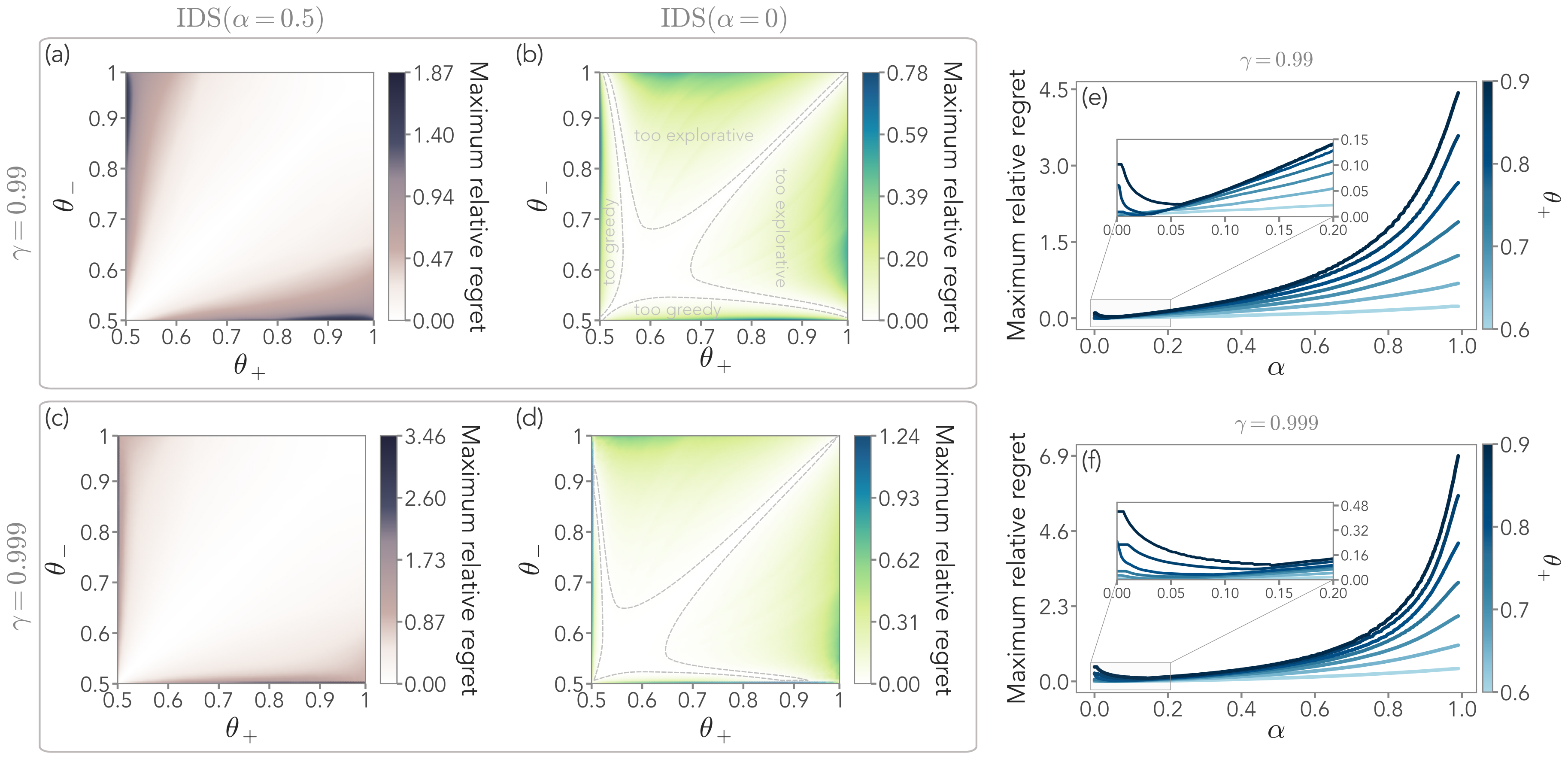}
    \caption{Performance of IDS($0$) and IDS($1/2$) in the asymmetric bandit problem. (a-b) Maximum relative distance to the optimal regret ($\Delta \mathcal{R}(\alpha)$) for IDS with $\alpha = 1/2$ (a) and $\alpha = 0$ (b), with $\gamma = 0.99$. The latter case performs better overall, although there are regions where the IDS($0$) policy is either too greedy or too explorative compared to the optimal one. The dashed lines correspond to the curves where $\Delta \mathcal{R}(\alpha)$ crosses a threshold of $10^{-2}$. Note that the range of $\Delta \mathcal{R}(\alpha)$ is smaller for IDS($0$). (c-d) Same, but for $\gamma = 0.999$. With respect to the previous case, the problem becomes harder, and the region where IDS($0$) is near-optimal shrinks and moves towards the axes. (e-f) For a specific choice of parameters (here $\theta_- = 0.55$), there exists an optimal value of $\alpha$ that minimizes the maximum relative regret.}
    \label{fig:figure3_placeholder}
\end{figure}

\subsubsection{Comparison between IDS$(0)$ and IDS$(1/2)$}
\noindent We begin by comparing $\Delta \mathcal{R}(\alpha)$ for $\alpha = 0$ with the standard choice of $\alpha = 1/2$ \cite{russo2014learning}. We show the maximum relative distance to the optimal regret for these two cases in Figure \ref{fig:figure3_placeholder}a-b for $1-\gamma = 10^{-2}$. As expected, both choices lead to a vanishing relative regret along the diagonal, the symmetric problem, while the worst performance is achieved on the axes corresponding to the fair coin case. In general, however, we find that IDS($0$) outperforms IDS($1/2$) in most of $(\theta_-, \theta_+)$ parameter space.

In particular, for $\alpha=1/2$, the IDS performance becomes gradually worse as we move from the diagonal to the axes, i.e., as the difference between the biases of the two coins increases. For $\alpha=0$, instead, the structure is different. We find that there is a region corresponding to nontrivial combinations of $\theta_\pm$ where performance is again optimal. Above this, performance slightly worsens and the maximum relative regret shows two local maxima around $\theta_\pm\to1, \theta_\mp \in [0.6,0.7]$, where IDS($1/2$) is the better choice. Nevertheless, samples of specific instances show that $\alpha=1/2$ systematically produces policies that are too greedy. With $\alpha=0$, on the other hand, the curve where the performance is optimal separates the space into the region above where it becomes too explorative and the region below where it is too greedy. When $\gamma$ is increased, this near-optimal region tends to shrink towards the axes -- however, since the axes correspond to the one fair coin case, it never touches them. Although for larger $\gamma$ the problem naturally becomes harder, IDS($0$) often remains advantageous with respect to the standard IDS($1/2$).

\subsubsection{Optimal $\alpha$ for different $\theta\pm$}
\noindent In general, we find that there exists an optimal $\alpha$ for a given pair of $(\theta_-, \theta_+)$. Indeed, although IDS can never lead to exploration-only policies because of the numerator in Eq.~\eqref{eq:inforatioalpha}, heuristically we can think of $\alpha$ as balancing between more greedy policies (large $\alpha$) and more explorative ones (small $\alpha$).

Figures \ref{fig:figure3_placeholder}e-f show $\Delta \mathcal{R}(\alpha)$ as a function of $\alpha$ and for two different values of $\gamma$. In general, we find that a smaller $\alpha$ tends to yield better performance. However, the relative regret does not always decrease monotonously with $\alpha$. In particular, when $\gamma$ increases, the minima of the regret for the same $\theta_\pm$ may be shifted to higher values $\alpha$. When this happens, it stems from IDS($0$) becoming too explorative as the effective horizon increases and information for future reward becomes more valuable. This is due to the fact that IDS is defined from the one-step information ratio, whereas the optimal policy is computed via the entire Bellman equation, resulting in different sensitivities to $\gamma$.

\section{Discussion}
\noindent In this work, we focused on a two-state Bernoulli bandit problem, where an agent must choose between two coins with different winning probabilities $\theta_\pm$, to study how the tradeoff between exploration and exploitation is handled in optimal decision making. To this end, based on previous results \cite{russo2014learning}, we introduced IDS($\alpha$), a class of heuristic strategies to explicitly balance the immediate regret and information gain at each step. In particular, we modified the IDS information measure to be better suited for the discounted infinite horizon setting, and introduced an explicit tuning parameter $\alpha$ that affects the relative importance of exploration and exploitation and that allowed us to derive a corresponding bound on the total regret. We first studied two special cases of the two-state problem. In the symmetric case, where the Bernoulli parameters of the reward distribution -- i.e., the winning probability of the two coins -- are identical, the greedy policy is optimal and IDS correctly recovered it as the expected information gain does not depend on the action in this case. In the fair coin case, where one coin is fair and the other is not, IDS($0$) performed sub-optimally but produced regrets of the order of the optimal one. Furthermore, in both these cases, we were able to derive either exact or approximate solutions to the Bellman's equation for the optimal policy, showing that in the symmetric problem, the regret converges to a constant as $\gamma \to 1$, whereas it diverges logarithmically for the fair coin one. 
%, which is not necessarily the case of the bounds (but we don’t really talk about this).

Among the open questions, it would be interesting to understand the dependence of the optimal
$\alpha$  on $\theta_{\pm}$ and $\gamma$ and interpret it in terms of the balance between exploitation and exploration.

Several extensions of IDS are possible. A key assumption in Eq.~\eqref{infofunction} is that information is defined by means of the Shannon entropy, but more general formulations could be obtained by using other measures of uncertainty, such as the Rényi entropy \cite{renyi1961measures,zimmert2021tsallis}. 
The potential of IDS goes far beyond the simple bandit problems discussed here. Very similar results can be obtained for the much wider class of POMDP where the actions of the agent do not affect the environment, for example non-stationary bandits. The full case with feedback by the agent on the environment introduces a new level of complexity and calls for a different definition of information gain. The study of these problems in minimal environments like the two-state one can help shed light on the intimate connection between information-theory and decision-making.

\clearpage
\appendix

\section{Value Iteration}
\label{app:value_iteration}
\noindent In this Appendix, we present a standard algorithm to solve the Bellman equation known as value iteration. For illustration, we describe the algorithm for a stationary problem with one dimensional belief $b(s)$ parametrized by a scalar $\beta$, e.g., $b(s) = (1 + s \beta) / 2$, although the algorithm and its convergence can be generalized to higher dimensional beliefs and to non-stationary settings.

\textcolor{magenta}{
%Unless otherwise specified, in this and the following, we consider a setting where the system may not be stationary and the state of the environment evolves according to a stochastic process characterized by $p(s'|s,a)$.
}
\\ \\
Let us define the Bellman operator $\mathsf{B}$ which acts nonlinearly on functions as
\begin{equation}
    \mathsf{B} W (\beta) = \max_a\left[r_a^Tb\,+\,\gamma  \sum_yp_b(y|a)W(\beta'_{ay})\right]
\end{equation}
where $\beta$ is the scalar parameterizing the belief, and $r_a^Tb$ is the expected immediate reward,
\begin{equation}
    r_a^Tb=\sum_s b(s)r(s,a)=\sum_s b(s)\mathbb{E}[R|s,a] \; ,
\end{equation}
and $p_b(y|a)$ is the probability to get observation $y$ when taking action $a$ at belief $b$:
\begin{equation}
    p_b(y|a)=\sum_s b(s)p(y|s,a)
\end{equation}
so that the Bellman equation is written as
\begin{equation}
    \mathsf{B}V^*=V^* \; .
\end{equation}
The value iteration algorithm aims at finding successive approximations of the optimal value function by applying the Bellman operator repeatedly, starting from an initial guess of the value function $V_0$, and is described in algorithm \ref{alg:vi}.

We now show that value iteration converges linearly in the max-norm, with coefficient $\gamma$. We recall that the $\infty$-norm of a function $f(\beta)$ is defined as
\begin{equation}
    ||f||_\infty = \sup_\beta |f(\beta)|
\end{equation}
and for two vectors $x, y$, we have the identity
\begin{equation}
    |\max_i x_i - \max_i y_i| \le \max_i |x_i-y_i| \; .
\end{equation}
As a first step, we prove that $\mathsf{B}$ is contracting in the simplex with contraction coefficient $\gamma<1$. Indeed
\begin{align}
    ||\mathsf{B}U-\mathsf{B}W||_\infty & = \sup_b \left|\max_a \left[r_a^Tb\,+\,\gamma  \sum_yp_b(y|a)U(\beta'_{ay})\right] -\max_a \left[r_a^Tb\,+\,\gamma  \sum_yp_b(y|a)W(\beta'_{ay})\right] \right| \nonumber \\
    & \le \gamma \sup_b \max_a \left|\sum_yp_b(y|a)(U(\beta'_{ay})- W(\beta'_{ay}))\right| \nonumber \\
    & \le \gamma \sup_b\max_{a} \sum_yp_b(y|a)\left|U(\beta'_{ay})- W(\beta'_{ay})\right| \nonumber \\
    & \le \gamma \sup_b\max_{a} \sum_yp_b(y|a)||U-W||_\infty =\gamma ||U-W||_\infty \; .
\end{align}
From contractivity, it follows that
\begin{equation}
    ||V_{k+1}-V^*||_\infty =||\mathsf{B}V_{k}-\mathsf{B}V^*||_\infty \le \gamma ||V_k-V^*||_\infty
\end{equation}
which, upon iteration, gives
\begin{equation}
    ||\mathsf{B}V_k - V^*||_\infty \le \gamma^k ||V_0-V^*||_\infty
\end{equation}
proving the linear convergence of value iteration with coefficient $\gamma$, for any initial guess $V_0$.

\begin{center}
\begin{minipage}{0.5\textwidth}
\begin{algorithm}[H]% or [t]/[b] as you prefer
\caption{Value iteration}\label{alg:vi}
\begin{algorithmic}
\State $V_0,\delta>0$\Comment{Initialization}
\Repeat
\State $V_{k+1}  = \mathsf{B}V_k$ \Comment{Iteration}
\Until $||V_{k+1}-V_k||_\infty <\delta$ \Comment{Termination}
\end{algorithmic}
\end{algorithm}
\end{minipage}
\end{center}

\section{Derivation of the regret bound for IDS}
\label{app:ids_bounds}
\noindent IDS was originally introduced in \cite{russo2014learning}, where a bound on the total regret was derived in the undiscounted finite horizon setting. However, if we use the mutual information directly as the denominator of the information ratio, this specific bound on the regret does not hold in the infinite-horizon discounted case. This motivates the introduction of a new information function better suited to this setting, for which a similar bound can be found. Furthermore, the following derivation uses a generalized information ratio containing a tuning parameter $\alpha$. This has previously been introduced \cite{lattimore2021mirror} and shown to produce an $\alpha-$dependent bound. The expression of the information function given here can generally be used to obtain a bound that holds also for non-stationary problems, as long as there is no feedback on the environment -- i.e., the selection of the action does not affect the transition probabilities between states of the environment. Therefore, in general, in this Appendix, we will not assume stationarity.

The total regret of a policy $\pi$ is defined as 
\begin{align}            
\mathcal{R}(\pi,b)&=\bar{v}(b)-v_\pi(b)
\end{align}
where the optimal return is given by
\begin{gather}
    \bar{v}(b) =\sum_s b(s) V_{\mathrm{MDP}}^*(s), \\
    V_{\mathrm{MDP}}^*(s)=\underset{a}{\text{max}} \;Q_{\mathrm{MDP}}^*(s,a)=\underset{a}{\text{max}}\left\{r_a(s)+\gamma\sum_{s'}T(s'|s)V_{\mathrm{MDP}}^*(s')\right\}
\end{gather}
where $Q_{\mathrm{MDP}}^*(s,a)$ is the optimal state-action value function, which corresponds to the value if the agent selects action $a$ and subsequently acts optimally. 

The value of the policy $\pi$ satisfies the Bellman equation (Eq.~\eqref{BellmanEqPOMDP}), i.e.,
\begin{align}
    v_\pi(b)&=\sum_{a}\pi(a|b)\left[r_a^Tb+\gamma \sum_y p_b(y|a)v_\pi(b_{ay}')\right]
\end{align}
where $p_b(y|a) = \sum_s b(s) p(y|s,a)$ is the probability of receiving an observation $y$ following action $a$ given the belief $b$ about the state of the environment, and $r_a^Tb$ is the expected immediate reward,
\begin{equation}
    r_a^Tb=\sum_s b(s)r(s,a)=\sum_s b(s)\mathbb{E}[R|s,a] \; ,
\end{equation} 
Thus,
\begin{align}
    v_\pi(b)&= \bar{v}(b)-\mathcal{R}(\pi,b) \notag \\
    v_\pi(b)&=\sum_a\pi(a|b)\left[ r_a^Tb+\gamma\sum_y p_b(y|a)[\bar{v}(b_{ay}')-\mathcal{R}(\pi,b_{ay}')]\right]
\end{align}
and substituting this back in the definition of the regret leads to
\begin{align}    
    \mathcal{R}(\pi,b)&=\bar{v}(b)-v_\pi(b) \notag \\
    &=\sum_a\pi(a|b)\left[\bar{v}(b)-(r_a^Tb+\gamma\sum_yp_b(y|a)[\bar{v}(b_{ay}')-\mathcal{R}(\pi,b_{ay}')])\right] \notag \\
    &=\sum_a \pi(a|b) \left[ \bar{v}(b)-[r_a^Tb+\gamma\sum_yp_b(y|a)\bar{v}(b_{ay}')]\right]+\gamma\sum_y p_b(y|a) \mathcal{R}(\pi,b_{ay}') \notag \\
    &=\Delta(\pi,b)+\gamma\sum_{ay}\pi(a|b)p_b(y|a)\mathcal{R}(\pi,b_{ay}')
\end{align}
where the quantity $\Delta(\pi,b):=\sum_a\pi(a|b)b^T(V^*-Q_a^*)\geq0$ is called the \textit{sub-optimality gap}. This gives a recursive relation expressing the regret as a discounted sum of one-step costs $\Delta(\pi,b)$. 

Let now $C_f(\pi,b)$ be a generic cumulative discounted cost incurred by the policy $\pi$ when the one step costs are $f(\pi,b)$:
\begin{equation}
    C_f(\pi,b)=f(\pi,b)+\gamma\sum_{ay}\pi(a|b)p_b(y|a)C_f(\pi,b_{ay}')
\end{equation}
with $0<\gamma<1$. Then, the following lemma holds.
\begin{lemma}
For any pair of immediate one step costs $f\geq0$, $g\geq0$ and $0<\alpha<1$,
\begin{equation}
    C_f(\pi,b)\leq C_{f^{1/\alpha}g^{1-1/\alpha}}(\pi,b)^\alpha C_g(\pi,b)^{1-\alpha}
\end{equation}
\end{lemma}

\begin{proof}
The proof of this Lemma is an application of Young's inequality, 
\begin{equation}
    uv\leq\alpha u^{1/\alpha}+(1-\alpha)v^{1/(1-\alpha)}
\end{equation}
for any $u,v\in\mathbb{R}_+$. Choosing $u=\lambda^{\alpha-1}g(\pi,b)^{\alpha-1}f(\pi,b)$ and $v=\lambda^{1-\alpha}g(\pi,b)^{1-\alpha}$ gives
\begin{equation}
    f(\pi,b)\leq\alpha\lambda^{1-1/\alpha}g(\pi,b)^{1-1/\alpha}f(\pi,b)^{1/\alpha}+(1-\alpha)\lambda g(\pi,b)
\end{equation}
for all $\lambda>0$. Thanks to the linearity of the recursion equations in $C$, this implies for any $\pi$ and $b$
\begin{equation}
    C_f(\pi,b)\leq\alpha\lambda^{1-1/\alpha}C_{f^{1/\alpha}g^{1-1/\alpha}}(\pi,b)+(1-\alpha)\lambda C_g(\pi,b) \; .
\end{equation}
Taking the minimum over $\lambda$ on the right-hand side, one obtains the result.
\end{proof}

Consider now a positive function $\Theta(\pi,b)>0$, which for us is going to represent a generic measure of information. Then, the information ratio of parameter $\alpha$ and its supremum are defined by
\begin{align}
    \Psi_\alpha(\pi, b)&=\frac{\Delta(\pi, b)^{1/\alpha}}{\Theta(\pi, b)^{1/\alpha-1}}, \quad
    \Psi_\alpha(\pi)=\underset{b}{\text{sup}}\Psi_\alpha(\pi, b) \; .
    \label{inforatioalpha}
\end{align}
Taking $f=\Delta$ (the one-step regret) and $g=\Theta$ in the lemma gives an upper bound on the regret
\begin{equation}
    \mathcal{R}(\pi,b)=C_\Delta\leq\left(C_{\Delta^{1/\alpha} \Theta^{1-1/\alpha}}\right)^\alpha C_\Theta^{1-\alpha}\leq\left(\frac{1}{1-\gamma}\Psi_\alpha\right)^\alpha C_\Theta^{1-\alpha}
    \label{regretboundalpha}
\end{equation}
by definition of $\Psi_\alpha(\pi)$. For the bound to depend on the policy only through this ratio, one has to find an information function for which $C_\Theta$ is independent of the policy. Let $H(b)$ be the Shannon entropy of the belief distribution $b$. We have
\begin{equation}
    \Theta(\pi,b)=(1-\gamma)H(b)+\gamma\sum_{ay}\pi(a|b)p_b(y|a)D_{-H}(b_{ay}',\hat{b}_a)\geq0
\end{equation}
where $D_{-H}$ is the Bregman divergence associated with $H$,
\begin{equation}
    D_{-H}(p,q)=-H(p)+H(q)+\langle\nabla H(q),p-q\rangle \; .
\end{equation}
The second term then becomes
\begin{align}
    &\gamma\sum_{ay}\pi(a|b)p_b(y|a)\left[-H(b_{ay}')+H(\hat{b}_a)+\langle\nabla H(\hat{b}_a),b_{ay}'-\hat{b}_a\rangle \right] \notag \\
    =\;&\gamma\sum_{ay}\pi(a|b)p_b(y|a)\left[H(\hat{b}_a)-H(b_{ay}')\right] +\gamma\sum_a\pi(a|b)\langle\nabla H(\hat{b}_a),\sum_yp_b(y|a)b_{ay}'-\hat{b}_a\rangle \notag \\
    =\;&\gamma\sum_{ay}\pi(a|b)p_b(y|a)\left[H(\hat{b}_a)-H(b_{ay}')\right]
\end{align}
where the second term vanishes because, if we recall that
\begin{equation}
    b_{ay}'(s')=p_b(s'|y,a)=\frac{\sum_s p(s'y|s,a)b(s)}{\sum_s p(y|s,a)b(s)}, \quad p_b(y|a)=\sum_s p(y|s,a)b(s)
\end{equation}
then
\begin{align}
\sum_y p_b(y|a) b_{ay}'=\sum_y \sum_s p(s'y|s,a)b(s) =\sum_s p(s'|s,a)b(s)=\hat{b}_a(s') \; .
\end{align}
Overall, we arrive at 
\begin{align}
    \Theta(\pi,b)&=(1-\gamma)h(b)+\gamma\sum_{ay}\pi(a|b)p_b(y|a)\left[H(\hat{b}_a)-H(b_{ay}')\right] \notag\\
    &= \overbrace{\gamma\sum_a \pi(a|b)\left(H(\hat{b}_a)-H(b)\right)}^{k(b)} +  \overbrace{H(b)-\gamma\sum_{ay} \pi(a|b)p_b (y|a)H(b_{ay}')}^{g(b)} \; .
    \label{infofunction}
\end{align}
In the cumulative sum $C_\Theta (\pi,b)$, the second part $g(b)= H(b)-\gamma\sum_{ay} \pi(a|b)p_b (y|a)H(b_{ay}')$ is a telescoping sum:
\begin{align}
    C_g(b)&=C_H(b)-\gamma\sum_{ay}\pi(a|b)p_b(y|a)C_H(b_{ay}') \notag \\
    C_H(b)&=H(b)+\gamma\sum_{ay}\pi(a|b)p_b(y|a)C_H(b_{ay}') \notag \\
    &\Rightarrow \; C_g(b)=H(b).
\end{align}
where the linearity of the cumulative discounted cost was used in the first equality. 

In the stationary case, we have $\hat{b}_a=b$ so that the first term $k(\pi,b)$ in $\Theta(\pi,b)$ vanishes and the bound in Eq.~\eqref{regretboundalpha} gives 
\begin{equation}
    \text{Regret}(\pi,b)\leq\left(\frac{\Psi_\alpha}{1-\gamma} \right)^\alpha H_0^{1-\alpha}
\end{equation}
where $H_0=H(b)$ is the initial entropy. 

Finally, the case $\alpha=0$ requires special attention as the exponent in the information ratio, Eq.~\eqref{inforatioalpha}, diverges. To avoid issues in the implementation of the IDS policy, we can raise the ratio to the positive exponent $\frac{\alpha}{1-\alpha}$, giving the equivalent minimization problem
\begin{equation}
    \pi_\mathrm{IDS}(b):=\underset{\pi}{\text{argmin}}\; \frac{\Delta^{\frac{1}{1-\alpha}}(\pi,b)}{\Theta(\pi,b)} \; .
\end{equation}
The bound $\alpha=0$ will be compared to the regrets obtained with this policy. To calculate it, we use the original information ratio, giving the bound
\begin{equation}
\text{Regret}(\pi,b)\leq\left(\frac{\Psi_\alpha}{1-\gamma} \right)^\alpha H_0^{1-\alpha}=\underset{b}{\text{sup}}\frac{\Delta(\pi,b)}{\Theta(\pi,b)}H_0
\label{eq:alpha0bound}
\end{equation}
which can be implemented immediately.

\section{Analytical derivations}
\label{app:analytics}
\subsection*{Symmetric Case}
\noindent In the following, $b\equiv b(s)$ is used to denote the belief distribution, while $\beta$ is its parametrization in the two-state problem (Eq.~\eqref{eq:belief_parametrization}). We recall that the optimal value function is defined by
\begin{align}
    v^*(b)&=\max_a\Biggl\{r_a^Tb\,+\,\gamma  \sum_yp_b(y|a)v^*(b'_{ay})\Biggr\}
\end{align}
where $r_a^Tb$ is the expected immediate reward,
\begin{equation}
    r_a^Tb=\sum_s b(s)r(s,a)=\sum_s b(s)\mathbb{E}[R{s,a}]
\end{equation}
and $p_b(y|a)=\sum_sb(s)p(y|s,a)$. Due to the symmetry of the problem, it can be assumed that for $\beta<0$ -- the region where action $a = -1$ is believed to be better by the agent -- the maximum is obtained for $a=-1$, and the opposite is true for $\beta>0$. Therefore, the equation can be solved separately for the left and right branches of $v^*(b)$, after which the two solutions will be matched.

For the left side, the equation becomes 
\begin{align}
    v^*_{-}(b)&=r_{-1}^Tb\,+\,\gamma  \sum_yp_b(y|a=-1)v_-^*(b'_{-1y}) \notag \\
\end{align}    
and we make the following ansatz:
\begin{equation}
    v^*_-(b)=\frac{r_a^Tb}{1-\gamma}+\phi(b)
\end{equation}
that gives
\begin{align}
    \frac{r_{-1}^Tb}{1-\gamma}+\phi(b)&=r_{-1}^Tb\,+\,\gamma  \sum_yp_b(y|-1)\left[\frac{r_{-1}^Tb'_{ay}}{1-\gamma}+\phi(b'_{-1y})\right] \notag\\
    &=r_{a_1}^Tb\,+\,\gamma  \sum_yp_b(y|-1)\left[\frac{r_{-1}^T\frac{p(y|s,a) b(s)}{p_b(y|-1)}}{1-\gamma}+\phi(b'_{ay})\right] \notag \\
    &=r_{-1}^Tb\,+\,\gamma\sum_y  \frac{r_{-1}^Tp(y|s,-1) b(s)}{1-\gamma}+\sum_yp_b(y|-1)\phi(b'_{-1y}) \notag\\
    &=\frac{r_{-1}^Tb}{1-\gamma}+\sum_yp_b(y|-1)\phi(b'_{ay}) \; .
\end{align}
So we are left with 
\begin{equation}
    \phi(b)=\gamma\sum_yp_b(y|b,a_1)\phi(b'_{ay})
\end{equation}
and we make another ansatz,
\begin{equation}
    \phi(b)=C\prod_{s} b(s)^{\zeta(s)}
\end{equation}
with $\sum_s\zeta(s)=1$. If we temporarily ignore the constant, we have that
\begin{align}
    \phi(b)\propto\prod_{s} b(s)^{\zeta(s)}&= \gamma\sum_yp_b(y|-1)\prod_{s} b_{-1y}(s)^{\zeta(s)} \notag\\ 
    &=\gamma\sum_y\prod_{s}p_b(y|-1)\left(\frac{p(y|s,a) b(s)}{p_b(y|-1)}\right)^{\zeta(s)} \notag\\
    &=\gamma\sum_y\prod_{s}p(y|s,-1)^{\zeta(s)}b(s)^{\zeta(s)} \notag\\
    &=\gamma\:\phi(b)\sum_y\prod_{s}b(s)^{\zeta(s)} \notag\\
    &\iff \sum_y\prod_{s}p(y|s,-1)^{\zeta(s)}=1/\gamma
\end{align}
and if we let $\zeta(s=-1)\equiv \zeta=1-\zeta(s=1)$, for this problem we obtain:
\begin{align}
    (1-\theta)^\zeta\theta^{1-\zeta}+\theta^\zeta(1-\theta)^{1-\zeta}=1/\gamma \; .
    \label{exponent_eq}
\end{align}
The possible solutions for $\zeta$ are given by 
\begin{equation}
    \zeta_{\pm}=\frac{\ln \left( \frac{1 \pm \sqrt{1 - 4\gamma^2\theta + 4\gamma^2\theta^2}}{2\gamma\theta} \right)}
{\ln(\frac{1 - \theta}{\theta})} \; .
\end{equation}
The correct exponent is chosen so as to satisfy the convexity and boundary conditions of the value function. The boundary conditions require that at $\beta=-1$ (when the system is in state $-1$ with probability 1) the value function is equal to that of state -1 in the underlying MDP, $\theta/(1-\gamma)$. The term $r_{-1}^Tb/(1-\gamma)$ in the solution goes exactly to this value at the boundary, so the correct exponent is $\zeta_+$, as $\phi(b)$ then goes to zero at $\beta=-1$.

When solving for the right branch of the solution, the term $r_1^Tb/(1-\gamma)$ replaces $r_{-1}^Tb/(1-\gamma)$ in the ansatz and one obtains the same equation (\ref{exponent_eq}). By the same argument as before, the correct exponent is now given by $\zeta_-$ in this case. It remains to determine the multiplicative constant $C$. At the decision boundary, which is known to be $\beta=0$ out of symmetry, the Bellman equation becomes
\begin{align}
    v^*(0)=1/2+\gamma v^*(2\theta-1)
\end{align}
since the value function at the possible updated beliefs must be equal, and these updated beliefs are $\pm 2\theta \mp1$. The optimal MDP value in the symmetric problem is given by
\begin{equation}
    \bar{v}(\beta)\equiv\frac{\theta}{1-\gamma}
\end{equation}
and the optimal regret is obtained by subtracting the solution for the value function from this expression. 

By expanding around $1-\gamma$, we obtain
\begin{equation}
    \mathcal{R^*}(\beta=0)=\frac{1}{2\delta}-c(\theta)(1-\gamma)+o(1-\gamma)^2
\end{equation}
where the coefficient of the linear term in the analytical expansion becomes
\begin{align*}
    c(\theta)&=-\frac{(\theta -1)}{8 \theta ^2 (2 \theta -1)^3 z_1(\theta)^2} c'(\theta)
\end{align*}
with $z_1(\theta) = \log[(1 - \theta)/\theta]$ and
\begin{align*}
    c'(\theta) & = (2 \theta -3) (2 \theta  (2 \theta -3) (4 \theta +1)+7) z_2(\theta)^2+4 \theta  z_2(\theta) \left(3 z_3(\theta)-8z_1(\theta)-4 \theta  (2 (\theta -2) \theta +3) \log \left(8 (1-\theta)
   \theta ^2\right)\right)+ \\
   & \quad + z_2(\theta) (4 (z_2(\theta) + z_3(\theta))+13 z_1(\theta)+(4 \theta  (4 \theta  (\theta  (4 \theta -5)+6)-11)+9)
   z_3(\theta))-24 \theta ^3 z_3(\theta)^2 \\
   z_2(\theta) & = \log\left[2(1-\theta)\right] \\
   z_3(\theta) & = \log(2\theta) \; .
\end{align*}
Now the ansatz can be evaluated at $\beta=0$ and $\beta=2\theta-1$ and solved for $C$, giving
\begin{equation}
    C=\frac{\gamma (1 - 2\theta)^2}{(1 - \gamma) \sqrt{1 - 4\gamma^2\theta(1 - \theta)}} \; .
\end{equation}
An example of the regrets obtained with the analytical solution and numerically by value iteration for comparison are shown in Figure \ref{fig:AnalyticalSolution} for three instances of the symmetric bandit problem. The analytical solutions match the numerical ones only at specific belief points. This is due to the fact that in the last step of the analytical derivation, it was assumed that the initial belief is $\beta=0$. In practice, when the agent starts at some belief $\beta_0$, it can only visit a countable set of points which correspond to all the values of $\beta$ that can be obtained through a sequence of belief updates given the prior $\beta_0$. The points where the analytical solution matches the numerical one is the set of points that can be visited starting from $\beta_0=0$. 

\begin{figure}[t!]
    \centering
    \includegraphics[width=0.7\linewidth]{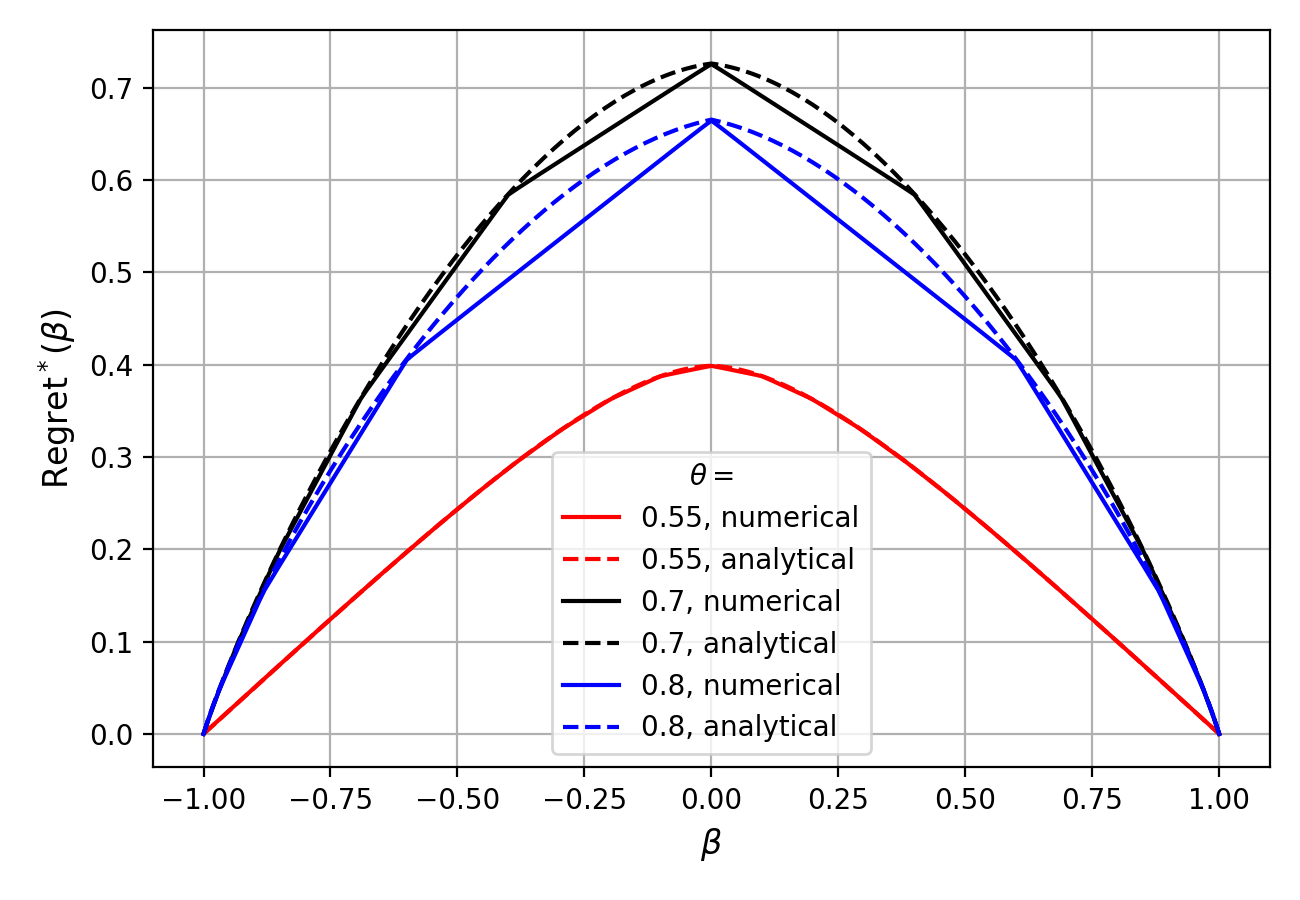}
    \caption{Regrets $\mathcal{R}(\beta)$ obtained numerically and analytically for the symmetric case with $\theta=0.55,0.7,0.8$}
    \label{fig:AnalyticalSolution}
\end{figure}

\subsection*{One Fair Coin Case}
\noindent In the case with one fair coin, the Bellman optimality equation is given by
\begin{align}
    v^*(b)&=\max\Biggl\{1/2+\gamma v^*(b'_{-1y}),\;r_a^Tb\,+\,\gamma  \sum_yp(y|b,a)v^*(b'_{1y})\Biggr\}
\end{align}
Where $b_{\pm1y}'$ is the updated belief when selecting action $\pm1$ and observing y. In the case where the fair coin is selected, $a=-1$, the belief update equation gives $b_{ay}'=b$ and the equation to solve for the left branch of the solution is given by
\begin{equation}
    v_-^*(b)=1/2+\gamma v^*_-(b)\iff v_-^*(b)=\frac{1}{2(1-\gamma)}
\end{equation}
For the right branch of the solution, we make the same Ansatz
as in the symmetric case:
\begin{equation}
    v_+^*(b)=\frac{r_{1}^Tb}{1-\gamma}+C\prod_s b(s)^{\zeta(s)} \;\;\;\;\;\;\text{with}\;\;\;\;\;\; \zeta(-1)=\zeta,\;\;\zeta(1)=1-\zeta
\end{equation}
we find the same solution for the exponent $\zeta_+$ as before. The constant can be determined by the matching condition at the decision boundary $\beta_c$
\begin{equation}
    v_-^*(\beta_c)=v_+^*(\beta_c)=1/2(1-\gamma)
\end{equation}
We now rewrite the Ansatz in terms of the parameter $\beta$ and impose the matching condition
\begin{equation}
    v^*(\beta_c)=\frac{1+\beta_c(2\theta-1)}{2(1-\gamma)}+C\frac{(1-\beta_c)^{\zeta}(1+\beta_c)^{1-\zeta}}{4}=\frac{1}{2(1-\gamma)}
\end{equation}
which gives
\begin{equation}
    C=-\frac{2\beta_c(2\theta-1)}{(1-\gamma)}\frac{1}{(1-\beta_c)^{\zeta}(1+\beta_c)^{1-\zeta}}
\end{equation}
We now ask that $v^{*'}(\beta_c)$ $\theta\to0.5$, which is approximately correct, with increasing precision as $\theta\to0.5$, we get the decision boundary
\begin{equation}
    \beta_c=-\frac{1}{2\zeta-1}<0
\end{equation}

\newpage
\section{Mutual information in the Symmetric Case}
\label{app_mutual}
\noindent The mutual information between the state and the observation given that the selected action is $a$ is
\begin{align}
    \mathrm{I}(S:Y|A=a)&=\sum_{sy} p(s,y|A=a)\log\left(\frac{p(s,y|A=a)}{p(s|A=a)p(y|A=a)}\right) \\
    &=\sum_{sy} b(s)p(y|s,a)\log\left(\frac{b(s)p(y|s,a)}{b(s)\left[\sum_{s'}b(s')p(y|s',a)\right]}\right) \\
    &=\sum_{sy}b(s)p(y|s,a)\log\left(\frac{p(y|s,a)}{\sum_{s'}b(s')p(y|s',a)}\right)
\end{align}
For the symmetric problem (figure \ref{fig:placeholderSchema2armedBB_POMDP}B, table 2), writing the terms in order s=-1,1 and y=0,1, we get
for action $-1$:
\begin{align}
    b(1-\theta)\log\left(\frac{b(1-\theta)}{b(1-\theta)+(1-b)\theta}\right)&+b\theta\log\left(\frac{b\theta}{b\theta+(1-b)(1-\theta)}\right) \\
    +(1-b)\theta\log\left(\frac{(1-b)\theta}{b(1-\theta)+(1-b)\theta}\right)&+(1-b)(1-\theta)\log\left(\frac{(1-b)(1-\theta)}{b\theta+(1-b)(1-\theta)}\right)
\end{align}
and for action $1$:
\begin{align}
    b\theta\log\left(\frac{b\theta}{b\theta+(1-b)(1-\theta)}\right)&+b(1-\theta)\log\left(\frac{b(1-\theta)}{b(1-\theta)+(1-b)\theta}\right) \\
    +(1-b)(1-\theta)\log\left(\frac{(1-b)(1-\theta)}{b\theta+(1-b)(1-\theta)}\right)&+(1-b)\theta\log\left(\frac{(1-b)\theta}{b(1-\theta)+(1-b)\theta}\right)
\end{align}

The terms are exactly the same but in different order, translating the fact that the information gained by losing with action $-1$ is the same as by winning with action $1$ and vice-versa.

\newpage
\bibliography{biblio}

\end{document}